\documentclass{article}

\usepackage[nonatbib,preprint]{neurips_2025}

\usepackage{multirow}

\usepackage[utf8]{inputenc} % allow utf-8 input
\usepackage[T1]{fontenc}    % use 8-bit T1 fonts
\usepackage{hyperref}       % hyperlinks
\usepackage{url}            % simple URL typesetting
\usepackage{booktabs}       % professional-quality tables
\usepackage{amssymb}
\usepackage{amsmath}
\usepackage{amsthm}
\usepackage{nicefrac}       % compact symbols for 1/2, etc.
\usepackage{microtype}      % microtypography
\usepackage[dvipsnames]{xcolor}         % colors
\usepackage{enumitem}
\usepackage{graphicx}
\usepackage{caption}
\usepackage{subcaption}
\usepackage{xfrac}
\usepackage{soul} % DELETE
\usepackage[defernumbers=true]{biblatex}

\newcommand{\defeq}{\stackrel{\text{def}}{=}}

\newcommand{\bbE}{\mathbb{E}}

\newcommand{\bbR}{\mathbb{R}}

\newcommand{\cL}{\mathcal{L}}
\newcommand{\cN}{\mathcal{N}}
\newcommand{\cS}{\mathcal{S}}
\newcommand{\cT}{\mathcal{T}}

\newcommand{\cA}{\mathcal{A}}

\newcommand{\cH}{\mathcal{H}}

\everypar{\looseness=-1}
\allowdisplaybreaks
\setlist[enumerate]{leftmargin=*}
\setlist[itemize]{leftmargin=*}

\newcommand{\E}[2][]{\bbE_{#1}\!\left[ #2 \right]}
\newcommand{\EX}[3][]{\tilde \bbE^{#2}_{#1}\!\left[ #3 \right]}
\newcommand{\EXD}[2][]{\EX[#1]{\alpha}{#2}}

\newtheorem{proposition}{Proposition}
\newtheorem*{proposition*}{Proposition}
\newtheorem*{theorem*}{Theorem}

\title{Risk-Averse Reinforcement Learning with Itakura-Saito Loss}

\author{%
  Igor Udovichenko\\
  Skolkovo Institute of Science and Technology\\
  Vega Institue Foundation\\
  Moscow, Russia\\
  \texttt{i.udovichenko@skoltech.ru} \\
  \And
  Olivier Croissant\\
  Natixis Foundation\\
  Paris, France\\
  \And
  Anita Toleutaeva\\
  Skolkovo Institute of Science and Technology\\
  Moscow, Russia\\
  \And
  Evgeny Burnaev\\
  Skolkovo Institute of Science and Technology\\
  Artificial Intelligence Research Institute\\
  Moscow, Russia\\
  \And
  Alexander Korotin\\
  Skolkovo Institute of Science and Technology\\
  Artificial Intelligence Research Institute\\
  Moscow, Russia\\
  \texttt{a.korotin@skoltech.ru} \\
}

\begin{document}

\maketitle

% \begin{abstract}
% Risk-averse reinforcement learning finds application in various high-stakes fields.
% Unlike classical reinforcement learning, which aims to maximize expected returns, risk-averse agents choose policies that minimize risk, occasionally sacrificing expected value.
% These preferences can be framed through utility theory.
% \st{We focus on the specific case of the exponential utility function, where we can derive the Bellman equations and employ various reinforcement learning algorithms with few modifications.}
% \red{We focus on the exponential utility function, from which Bellman equations can be derived with few modifications to standard RL.} However, these methods suffer from numerical instability due to the need for exponent computation throughout the process.
% \st{To address this, we introduce a numerically stable and mathematically sound loss function based on the Itakura-Saito divergence for learning state-value and action-value functions.}
% \red{To address this, we introduce a numerically stable and mathematically sound loss function based on the Itakura-Saito (IS) divergence. Our IS-based approach preserves the risk-averse Bellman optimality criterion without incurring scale-dependent blow-ups.}
% We evaluate our proposed loss function against established alternatives, both theoretically and empirically.
% In the experimental section, we explore multiple financial scenarios, some with known analytical solutions, and show that our loss function outperforms the alternatives.
% \end{abstract}

\begin{abstract}
Risk-averse reinforcement learning finds application in various high-stakes fields.
Unlike classical reinforcement learning, which aims to maximize expected returns, risk-averse agents choose policies that minimize risk, occasionally sacrificing expected value.
These preferences can be framed through utility theory.
We focus on the specific case of the exponential utility function, where one can derive the Bellman equations and employ various reinforcement learning algorithms with few modifications.
To address this, we introduce to the broad machine learning community a numerically stable and mathematically sound loss function based on the Itakura-Saito divergence for learning state-value and action-value functions.
We evaluate the Itakura-Saito loss function against established alternatives, both theoretically and empirically.
In the experimental section, we explore multiple scenarios, some with known analytical solutions, and show that the considered loss function outperforms the alternatives.
\end{abstract}

\vspace{-2mm}
\section{Introduction}
\vspace{-2mm}
Reinforcement learning (RL) has achieved remarkable success in domains where the primary goal is to learn policies by interacting with an environment~\cite{li2018deepreinforcementlearning,sutton2018reinforcement}.
The goal is often formalized through a Markov decision process (MDP), which aims to find a policy that maximizes the expected cumulative reward received during the interaction with the environment~\cite{sutton2018reinforcement}.

However, agents must prioritize risk mitigation alongside performance in high-stakes applications such as finance, healthcare, and autonomous systems~\cite{bernoulli1751commentarii, von1947theory, kolm2019modern}.
Traditional risk-neutral RL frameworks, which optimize for expected returns, often fail to account for the variability and tail risks inherent in these settings.
Risk-averse RL addresses this by incorporating preferences that penalize uncertainty, typically formalized through utility theory or coherent risk measures~\cite{follmer2011stochastic}.

Among the various utility-based methods, the \emph{exponential} (or \emph{entropic}) utility function stands out for its convenient properties~\cite{penner2007dynamic, hau2023entropic}.
Yet, existing exponential-utility RL approaches typically require exponentiation of the value function at each step and can suffer from significant numerical instabilities~\cite{enders2024risk}.
These instabilities often prevent reliable convergence.

This paper introduces an approach to risk-averse RL by leveraging the Itakura-Saito (IS) divergence~\cite{itakura1968analysis}, a specific case of Bregman divergence~\cite{bregman1967relaxation, banerjee2005clustering} historically used in signal processing~\cite{chan2008advances} and non-negative matrix factorization~\cite{fevotte2009nonnegative}.
The loss was first introduced in~\cite{murray2022deep} yet remains unexplored in the broad RL community.
In the paper, we derive the loss from the fundamental property of the Bregman divergence and compare it empirically against known alternatives.

\paragraph{Contributions}
\begin{enumerate}
    \item \textbf{Derivation:} We formally derive the IS loss from Bregman divergence and show that it recovers the exponential utility’s Bellman equation under mild conditions, ensuring that the resulting value estimate is correct and that the method is scale-invariant. We also derive the corresponding stochastic approximation rule for the tabular setup.
    \item \textbf{Empirical Validation:} Across a range of benchmarks---from analytically tractable portfolio examples~\ref{sec:exp:toy} to a combinatorial RL problem~\ref{sec:exp:rssac}---the IS loss outperforms existing baselines.
\end{enumerate}

\vspace{-2mm}
\section{Background and Related Works}\label{sec:bg}
\vspace{-2mm}

In this section, we briefly review the essentials of reinforcement learning (RL), emphasizing how risk aversion arises in decision-making processes, and why exponential utility proves useful for risk-sensitive control.

\vspace{-2mm}
\subsection{Markov Decision Process (MDP)}
\vspace{-2mm}
Consider a MDP of the form $(\cS, \cA, r, p, s^0)$, where $\cS$ and $\cA$ are sets of states and actions, respectively.
Here $r(s, a, s')$ is the reward function, dependent on the current state, action, and next state.
State-transition probability (or density) is denoted as $p(s' \mid s, a)$.
The initial state at $t=0$ is $s^0$.
We assume the finite time horizon, so the time index $t=0,\ldots, T$, where $T < \infty$.
The discount factor $\gamma = 1$ for simplicity.
Extending our ideas on the case with $\gamma < 1$ and infinite time horizon is straightforward~\cite{hau2023entropic}.
The \textbf{timestamp is assumed to be a part of the state} to avoid notation overload.
By $\Pi$ we denote the set of Markov policies $\pi(a \mid s)$.
We restrict our considerations to the class of Markov policies, because the optimal policy lies in it~\cite{hau2023entropic}.

We define the trajectory $\cT^\pi$ as a random sequence of states and actions according to a policy $\pi$:
\begin{equation}
    \cT^\pi \defeq \left(s^0, a^0, s^{1}, a^{1}, \ldots, s^T \right), \quad a^t \sim \pi(\cdot \mid s_t),\quad t = 0,\ldots,T-1.
\end{equation}
A trajectory part started at state $s$ is denoted as $\cT^\pi_s$.
Furthermore, we write $\cT^\pi_{s, a}$ if the action at $s$ is also fixed rather than sampled from $\pi$.
We define the random return of a policy $\pi$ as follows:
\begin{equation}
\textcolor{Red}{R^\pi} \big/ \textcolor{Green}{R^\pi(s)} \big/ \textcolor{Blue}{R^\pi(s, a)}
= \sum_{\tau=t}^{T-1} r(s^\tau, a^\tau, s^{\tau+1}),\quad
\left( s^t, a^t,\ldots,s^T \right) \sim \textcolor{Red}{\cT^\pi} \big/ \textcolor{Green}{\cT^\pi_s} \big/ \textcolor{Blue}{\cT^\pi_{s, a}},
\end{equation}
where $t=0$ for $R^\pi$ or $t$ is a timestamp of $s$ for $R^\pi(s)$ and $R^\pi(s, a)$.

The standard goal of RL is to find a policy that maximizes the expected return:
\begin{equation}\label{eq:rl_goal}
\pi^* = \arg \max_{\pi \in \Pi} \E[\cT^\pi]{R^\pi}.
\end{equation}
% \red{We let $R^\pi$ denote the cumulative reward over an episode under policy $\pi$. In risk-neutral RL, $\pi^*$ is chosen to maximize $\mathbb{E}[R^\pi]$. However, such an objective does not capture agents’ aversion to variability or catastrophic losses.}

\vspace{-2mm}
\subsection{Learning Optimal Value Functions}
\vspace{-2mm}
Many RL algorithms rely on the state-value function $V^\pi(s)$ (or simply $V$-function) or action-value function $Q^\pi(s, a)$ ($Q$-function) defined as follows:
\begin{equation}\label{eq:vdef}
V^\pi(s)\defeq \E[\cT^\pi_s]{R^\pi(s)},\quad Q^\pi(s, a) \defeq \E[\cT^\pi_{s,a}]{R^\pi(s, a)},
\end{equation}

$V^\pi(s) = 0$ and $Q^\pi(s, a) = 0$ in all terminal states $s$.
We denote the $V$-function of the optimal policy $\pi^*$ (optimal value function) by $V^*(\cdot)$, and the optimal $Q$-function as $Q^*(\cdot, \cdot)$.
Thanks to the \emph{tower property of the conditional expectation} operator, value and optimal value functions satisfy the famous Bellman equations~\cite{sutton2018reinforcement}:
\begin{align}
    V^\pi(s) &= \E[a, s']{r(s, a, s') + V^\pi(s')}, \label{eq:vv}\tag{VV} \\
    V^*(s) &= \max_{a \in \cA} \E[s']{r(s, a, s') + V^*(s')}, \label{eq:vvopt}\tag{VV*} \\
    Q^\pi(s, a) &= \E[a', s']{r(s, a, s') + Q^\pi(s', a')}, \label{eq:qq}\tag{QQ} \\
    Q^*(s, a) &= \E[s']{r(s, a, s') + \max_{a' \in \cA} Q^*(s', a')}, \label{eq:qqopt}\tag{QQ*}
\end{align}
where the expectation is taken over the variables sampled from the policy or the state-transition law.

% The simplest example is the deep $Q$-learning algorithm.
% It aims to find the optimal $Q$-function in the class of neural network (NN) $Q_\theta$ parametrized by the weights $\theta$.
% The optimization procedure is the following:
% \begin{equation}\label{eq:dqn}
% \cL(\theta) = \E[(s, a, s')]{\frac12 \left(Q_\theta(s, a) - r(s, a, s') - \max_{a'} Q_{\theta^-}(s', a') \right)^2},
% \end{equation}
% where $\theta^-$ denotes the parameters of the target network~\cite{fan2020theoretical}.
% By $\E[(s, a, s')]{\cdot}$ we denote the expectation computed over the $(s, a, s')$-tuples, being collected during the interaction of the agent with the environment.
% Mean squared error (MSE) loss ensures that the minimum of $\cL(\theta)$ is the conditional expectation, the rhs of~\eqref{eq:qqopt}, so the objective indeed learns the optimal $Q$-function~\cite{mnih2013playing, fan2020theoretical}.

Many deep learning algorithms in RL involve learning either the $Q$- or the $V$-function using Bellman equations.
Policy gradient methods often rely on learning the value function for some policy $\pi$~\cite{li2018deepreinforcementlearning}.
Policy evaluation step~\cite{sutton2018reinforcement} aims to find a $V^\pi$.
The $V^\pi_\theta$ is a NN parametrized by a weight vector $\theta$.
The NN is trained by optimizing the MSE objective that regresses $V_\theta^\pi$ on the rhs of~\eqref{eq:vv}.
Define
\begin{equation}\label{eq:delta}
\delta_V(\theta) = V^\pi_\theta(s) - r(s, a, s') -  V^\pi_{\theta^-}(s'),
\end{equation}
the difference between the current approximation of the $V$-function and its target from the corresponding Bellman equation.
By $\theta^-$ we denote the target network's weights.
Then
\begin{equation}\label{eq:vl}\tag{MSE}
\cL^{\text{MSE}}(\theta) = \E[(s, a,s')]{\tfrac12 \delta_V(\theta)^2},
\end{equation}
is optimized when $V^\pi_\theta = V^*$.
It follows from the fundamental property of conditional expectation being the optimal $\cL^2$-predictor, so the minimum is attained when $V^\pi_\theta$ satisfies~\eqref{eq:vv}.
The expectation is taken over tuples $(s, a, s')$ collected during the interaction of an agent an environment.

% The agent interacts with the environment and collects the tuples $(s_i, a_i, s'_i)$.
% For each tuple, the target $y^i$ is computed as
% \[
% y_i = r(s_i, a_i, s'_i) + \max_{a'} Q_{\theta^-}(s', a'),
% \]
% where $\theta^-$ are the target network parameters updated less often than $\theta$.
% After the batch of tuples and targets is collected, the weights $\theta$ are updated using gradient descent:
% \begin{gather*}
% \cL(\theta) = \frac1N \sum_{i=1}^N \frac12 \left( Q_\theta(s_i, a_i) - y_i \right)^2, \\
% \theta \leftarrow \theta - \eta \nabla \cL(\theta),
% \end{gather*}
% where $\eta$ is the learning rate parameter of the stochastic gradient descent.

\vspace{-2mm}
\subsection{Formalizing the Risk Aversion}
\vspace{-2mm}
Consider two alternative returns an agent can choose, one is deterministic zero reward, and the other is either $1$ or $-1$ with equal probabilities.
For objective~\eqref{eq:rl_goal} they are equal, because their expected values are equal, but for some applications the deterministic reward is preferable, because it is ``less risky''.
There are many possible ways to formalize the preferences of random outcomes.
The most straightforward one is through the von Neumann–Morgenstern (VNM) utility theorem.
It states that under 4 VNM-rationality axioms~\cite{von1947theory, follmer2011stochastic}, the utility function can describe the agent's preferences, i.e., random outcome $X$ is preferable to $Y$, if $\E{u(X)} > \E{u(Y)}$.
The utility function is defined up to affine transformations, e.g. $u(\cdot)$ and $a + bu(\cdot)$ describe the same preferences for $a \in \bbR$ and $b > 0$.
A natural assumption, not implied by the VNM theorem, is that $u(\cdot)$ is a strictly increasing function, which can be interpreted as ``there is no such thing as too much money''.
Under this assumption, one can define the \emph{certainty equivalent (CE)} as $u^{-1}\!\left( \E{u(X} \right)$, a non-random reward that is equivalent to a random one from the VNM agent's point of view.

The exponential (also called entropic) utility function $u(x) = \alpha^{-1} (1-e^{-\alpha x})$ represents a significant specific example.
Coefficient $\alpha > 0$ defines the agent's \emph{risk aversion}.
In some applications, one can also consider the case $\alpha < 0$, in which the agent is said to be \emph{risk seeking}.
If $\alpha \rightarrow 0$, the agent becomes indifferent to risk and treats outcomes with the same expected values as equal in the limit.
As $\alpha \rightarrow \infty$, the agent treats all positive returns equally regardless of their magnitude and does not tolerate any losses.

The certainty equivalent (CE) for a random variable $X$ is defined as the guaranteed amount that an agent would accept instead of taking a risk. For exponential utility, this is expressed as
\begin{equation}
\EXD{X} \defeq -\alpha^{-1} \log \E{e^{-\alpha X}},
\end{equation}
where $\alpha>0$  is the risk aversion parameter.
Operator $\EXD{\cdot}$ shares many properties with expectation, hence the notation.
The key ones are~\cite{follmer2011stochastic}:
\begin{enumerate}[label=P.\,\arabic*,ref=P.\,\arabic*]
    \item\label{p:n} \emph{Normalization:} $\EXD{0} = 0$.
    \item\label{p:m} \emph{Monotonicity:} If $X \le Y\ \text{a.s.}$, then $\EXD{X} \le \EXD{Y}$.
    \item\label{p:tr} \emph{Translation invariance:} $\EXD{X + c} = \EXD{X} + c,\ c \in \bbR$.
    \item\label{p:tq} \emph{Tower property:} $\EXD{\EXD{Y \mid X}} = \EXD{Y}$
\end{enumerate}
Unlike the expectation, $\EXD{\cdot}$ is not linear, but concave, which is a weaker property:
\begin{enumerate}[resume,label=P.\,\arabic*,ref=P.\,\arabic*]
    \item\label{p:c} \emph{Concavity:} $\EXD{\lambda X + (1 - \lambda) Y} \ge \lambda \EXD{X} + (1 - \lambda) \EXD{Y},\ \lambda \in [0, 1]$.
\end{enumerate}
These unique properties allow us to derive Bellman equations~\cite{hau2023entropic} for the exponential utility similar to those widely used to solve risk-neutral MDPs.

\vspace{-2mm}
\subsection{Entropic MDP and its limitations}
\vspace{-2mm}
Risk-averse MDP aims to maximize the anticipated future return adjusted for unwillingness to bear excess risks.
Due to our focus on the exponential utility, we formalize the objective as follows:
\begin{equation}
    \pi^* = \arg \max_{\pi \in \Pi} \EXD[\cT^\pi]{R^\pi}.
\end{equation}
This objective is analogous to~\eqref{eq:rl_goal}, but the expectation operator $\E{\cdot}$ is replaced with the CE operator $\EXD{\cdot}$ of exponential utility.
We can define the value functions analogously to~\eqref{eq:vdef}:
\begin{equation}\label{eq:evdef}
\tilde V^\pi(s) \defeq \EXD{R^\pi(s)},\quad \tilde Q^\pi(s, a) \defeq \EXD{R^\pi(s, a)},
\end{equation}
The Bellman equations become~\cite{hau2023entropic}:
\begin{align}
    \tilde V^\pi(s) &= \EXD[a, s']{r(s, a, s') + \tilde V^\pi(s')}, \label{eq:evv}\tag{EVV} \\
    \tilde V^*(s) &= \max_{a \in \cA} \EXD[s']{r(s, a, s') + \tilde V^*(s')}, \label{eq:evvopt}\tag{EVV*} \\
    \tilde Q^\pi(s, a) &= \EXD[a', s']{r(s, a, s') + \tilde Q^\pi(s', a')}, \label{eq:eqq}\tag{EQQ} \\
    \tilde Q^*(s, a) &= \EXD[s']{r(s, a, s') + \max_{a' \in \cA} \tilde Q^*(s', a')}. \label{eq:eqqopt}\tag{EQQ*}
\end{align}

The seminal work on risk-sensitive MDP considered exponential utility~\cite{howard1972risk}.
Recently, works~\cite{borkar2001sensitivity, borkar2002q, borkar2002risk, mihatsch2002risk, nass2019entropic, fei2020risk, deletang2021model, fei2021risk, fei2021exponential, fei2022cascaded, moharrami2025policy, noorani2023exponential, hau2023entropic, enders2024risk, granadosrisk} also considered MDPs with exponential utility specifically.
% Also, the importance of exponential utility can be highlighted by the fact that its optimization appears as a dual problem in the martingale Schr{\"o}dinger bridge problem~\cite{henry2019martingale, nutz2023martingale}.
Many of these methods rely on learning the optimal $Q$- or $V$-function.
The value function is often auxiliary in RL algorithms, since the ultimate goal is to learn \emph{policy}.
However, in some applications, learning the precise value function is critical.
For example, in finance, it represents the portfolio value or the price of the derivative being hedged~\cite{deep_hedging, buehler2021deep, buehler2022deep, murray2022deep, kolm2019dynamic, halperin2020qlbs}.

Since the CE operator replaces the expectation, note that objective~\eqref{eq:vl} does not learn the correct value function for entropic MDP.
The majority of the works mentioned above rely on the following objective, which we call exponential MSE loss:
\begin{equation}\label{eq:exp_mse}\tag{EMSE}
\cL^{\text{EMSE}}(\theta) = \E[(s,a,s')]{\frac12 \alpha^{-2} \Bigl(
    \exp \bigl\{-\alpha \tilde V^\pi_\theta(s)\bigr\} -
    \exp \bigl\{ -\alpha r(s, a, s') - \alpha \tilde V^\pi_{\theta^-}(s') \bigr\}
\Bigr)^2}.
\end{equation}
The optimizer of this loss is a correct value function for the risk-averse MDP.
By the Taylor expansion~\eqref{eq:exp_mse} can be rewritten as:
\begin{equation}\label{eq:emse_taylor}
\cL^\text{EMSE} = \tfrac12 \exp \bigl\{-2 \alpha \tilde V^\pi_\theta(s)\bigr\} \E[a, s']{ \delta_{\tilde V}(\theta)^2 + o\bigl( \delta_{\tilde V}(\theta)^2 \bigr)},
\end{equation}
so, the objective~\eqref{eq:exp_mse} reduces to MSE loss for risk-tolerant agents or small error $\delta_{\tilde V}(\theta)$.
Note that it depends on $\theta$ not only through the $\delta_{\tilde V}(\theta)$ because of the factor $\exp \bigl\{-2 \alpha \tilde V^\pi_\theta(s)\bigr\}$.
We argue that such dependence is highly undesirable.
First, the loss vanishes for high positive values of $\tilde V^\pi_\theta(s)$ and explodes for high negative values.
Works~\cite{granadosrisk, enders2024risk} note its numerical instability.
Second, from the translation invariance property~\ref{p:tr} of $\EXD{\cdot}$, the learning of $\tilde V^\pi_\theta$ should not depend on the absolute levels of its values.

Another loss was proposed in~\cite{deletang2021model}, which we call \emph{softplus} because of the term $\log\bigl(1 + \exp \{\alpha z\}\bigr)$:
\begin{equation}\label{eq:soft_plus}\tag{SP}
    \cL^\text{SP}(\theta) =
    2 \delta_{\tilde V}(\theta)\alpha^{-1} \log \bigl(1 + \exp \bigl\{\alpha \delta_{\tilde V}(\theta)\bigr\} \bigr)
    +2 \alpha^{-2} \mathrm{li}_2 \bigl(-\exp \bigl\{\alpha \delta_{\tilde V}(\theta)\bigr\} \bigr)
    + \pi^2 \!\big/ (6 \alpha^2),
\end{equation}
where $\mathrm{li}_2(z)$ is Spence's dilogarithm function $\mathrm{li}_2(z) = -\int_0^z \frac{\log(1 - z)}z\,dz$.
It appears as an objective, whose gradient coincides with the heuristic stochastic approximation rule, introduced in~\cite{mihatsch2002risk}.
This loss depends on $\theta$ only through the $\delta_{\tilde V}(\theta)$, which makes it numerically more stable than~\eqref{eq:soft_plus}.
However, it is not convex and only learns the correct value function when the target has a Gaussian distribution.

In summary, the known losses have the following limitations:
\begin{itemize}
    \item \ref{eq:exp_mse} is numerically unstable,
    \item \ref{eq:soft_plus} is optimized by the value function only in a specific case.
\end{itemize}
In the following section, we consider the objective that addresses these limitations.

\vspace{-2mm}
\section{Itakura-Saito Loss for Learning Risk-Averse Value Function}\label{sec:method}
\vspace{-2mm}

\begin{figure}[t]
    \centering
    \includegraphics[width=\linewidth]{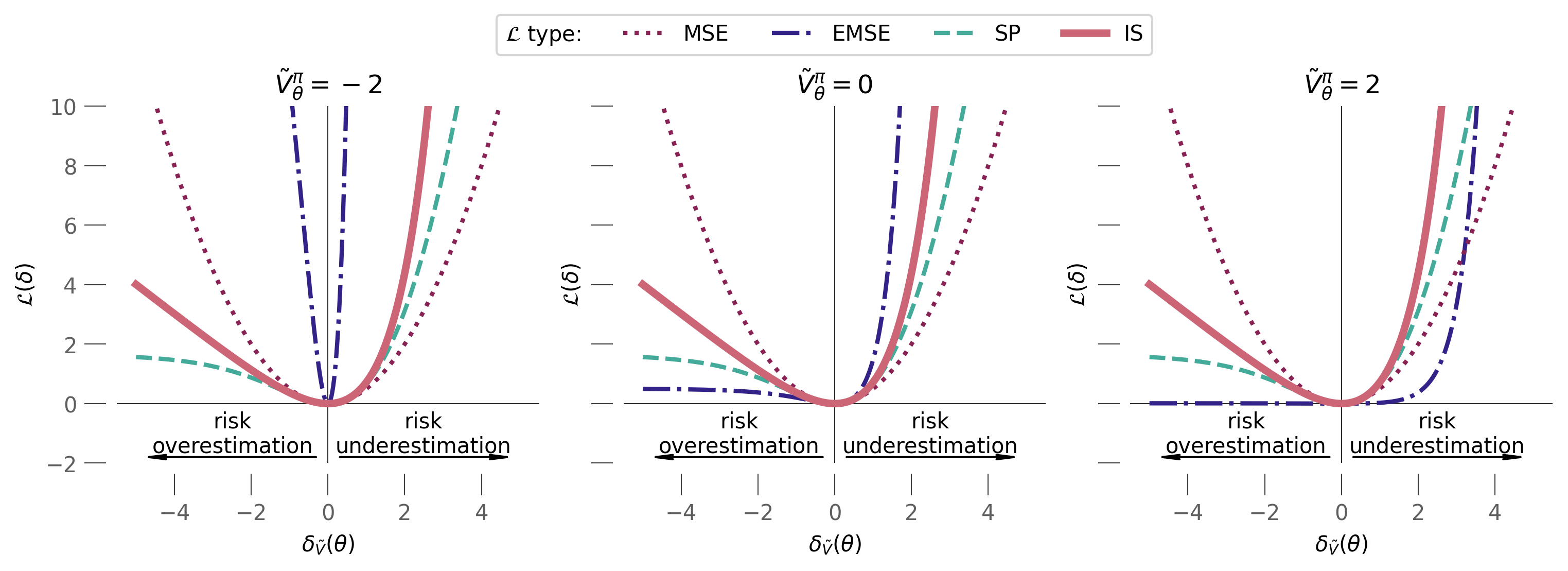}
    \caption{%
        % Losses comparison for $\alpha = 1$.
        % The CE of return is overestimated (and risk is underestimated) by $\tilde V^\pi_\theta$ where $\delta_{\tilde V}(\theta) > 0$, so risk-averse losses (all except MSE) penalize risk underestimation more than overestimation.
        % EMSE depends on the absolute level of $\tilde V^\pi_\theta$.
        % It causes numerical instability.
        Comparison of loss penalties for a one-step value prediction error $\delta_{\tilde V}(\theta)$ when $\alpha=1$.
        A positive $\delta_{\tilde V}(\theta) > 0$ means the current estimate $\tilde V_\theta(s)$ underestimates the true CE value (the return is higher than expected).
        Risk-averse losses heavily penalize underestimation ($\delta_{\tilde V}(\theta) > 0$) since underestimating the value implies unaccounted risk, whereas overestimation ($\delta_{\tilde V}(\theta)< 0 $) is penalized less.
        MSE, being risk-neutral, is symmetric.
        EMSE (exponential MSE) grows with the absolute value of $V$, leading to numerical instability for large values.
    }
    \label{fig:losses}
\end{figure}

In this section, we consider an objective for learning the value function that is \textbf{mathematically correct and numerically stable}.
While MSE loss minimizes the expectation, there are other objectives with similar properties, especially for risk-sensitive settings.

\vspace{-2mm}
\subsection{Bregman Divergence and Itakura-Saito Loss}
\vspace{-2mm}

Recall the definition of Bregman divergence (BD)~\cite{bregman1967relaxation}:
\begin{equation}\label{eq:breg}
    d_\varphi(x, y) = \varphi(x) - \varphi(y) - \bigl< x - y, \nabla \varphi(y) \bigr>,
\end{equation}
where $\varphi(\cdot)$ is a differentiable convex function.
It measures the discrepancy induced by a convex function $\varphi$;
The important property is that the true mean minimizes the expected divergence:
\begin{equation}
\E{X} = \arg \min_y \E{d_\varphi(X, y)}.
\end{equation}
In other words, the expectation is the ``best prediction'' under any Bregman loss, a generalization of the fact that the mean minimizes MSE.
Moreover, BD is an exhaustive class of loss functions for which the expectation is the optimizer~\cite{banerjee2005optimality}.
BD with $\varphi(z) = \frac12 \|z\|^2$ reduces to the MSE loss.

BD with $\varphi(z) = -\log z$ is known as \emph{Itakura-Saito (IS) distance}~\cite{itakura1968analysis} $d_\text{IS}(x, y) = x/y - \log(x/y) - 1$.
It is widely used in audio processing~\cite{chan2008advances} and non-negative matrix factorization~\cite{fevotte2009nonnegative}.
Substituting the prediction $\exp \bigl\{-\alpha \tilde V^\pi_\theta(s)\bigr\}$ and the regression target $\exp \bigl\{ -\alpha r(s, a, s') - \alpha \tilde V^\pi_{\theta^-}(s') \bigr\}$ into $d_\text{IS}(\cdot, \cdot)$ yields the following \textbf{Itakura-Saito loss}~\cite{murray2022deep}:
\begin{equation}\label{eq:is_loss}\tag{IS}
    \cL^\text{IS}(\theta) \stackrel{\text{def}}{=} \alpha^{-2} \E[(s, a, s')]{\exp\bigl\{ \alpha \delta_{\tilde V}(\theta) \bigr\} - \alpha \delta_{\tilde V}(\theta) - 1}.
\end{equation}
In Appendix~\ref{ap:proof} we \underline{formally state and prove} the following proposition.
\begin{proposition}\label{prop:is}
Under mild assumptions the value function that minimizes~\eqref{eq:is_loss} satisfies~\eqref{eq:evv}.
\end{proposition}
First, note that~\eqref{eq:is_loss} depends on $\theta$ only through $\delta_{\tilde V}(\theta)$.
Second, by the Taylor expansion:
\begin{equation}
\cL^\text{IS} = \E[(s, a, s')]{\tfrac12 \delta_{\tilde V}(\theta)^2 + o\left( \delta_{\tilde V}(\theta)^2 \right)},
\end{equation}
so, for a risk-tolerant agent or small discrepancies between the $V$-function and its target, the Itakura-Saito loss reduces to the MSE loss.
We compare visually all losses in Figure~\ref{fig:losses}.
Notably, IS loss casts the risk-sensitive Bellman criterion into a form suitable for stochastic gradient descent---circumventing the bias issues identified in past risk-sensitive Q-learning attempts~\cite{mihatsch2002risk}.

\vspace{-2mm}
\subsection{Stochastic Approximation Rule}
\vspace{-2mm}

Tabular RL algorithms~\cite{sutton2018reinforcement} rely on the stochastic approximation (SA)~\cite{robbins1951stochastic} procedure rather than optimizing some loss function.
The procedures are equivalent in the following sense:
\[
\theta^* = \arg \min_\theta \cL(\theta) \iff \nabla \cL(\theta^*) = 0.
\]
Solving the latter problem with SA is equivalent to solving the former with stochastic gradient descent.
Taking the gradient of~\eqref{eq:is_loss} we derive the following stochastic approximation scheme:
\begin{equation}\label{eq:st_app}
    \tilde V_{k+1}^\pi(s) \leftarrow \tilde V_k^\pi(s) - \eta_k \alpha^{-1} \left(\exp\bigl\{ \alpha \delta_{\tilde V}(k) \bigr\} - 1\right),
\end{equation}
where $\eta_k$ is the learning rate, $\tilde V^\pi_k(\cdot)$ is the approximation of $\tilde V^\pi(\cdot)$ on the $k$-th step of the stochastic approximation algorithm, and $\delta_{\tilde V}(k)$ is the difference between $\tilde V_k(s)$ and its target.

\vspace{-2mm}
\section{Experiments}\label{sec:exp}
\vspace{-2mm}

In this section, we empirically compare the~\eqref{eq:is_loss} loss against~\eqref{eq:exp_mse} and~\eqref{eq:soft_plus}.
We choose the financial problems as our primary experimental setups for the following reasons:
\begin{enumerate}
    \item The very concept of risk aversion originates in economics and finance~\cite{bernoulli1751commentarii, von1947theory, follmer2011stochastic}.
    \item RL is widely considered in financial literature~\cite{kolm2019modern, hambly2023recent}.
    \item The proposed setups admit ground truth \underline{analytical solutions} (see Appendix~\ref{ap:sol}) or theoretical references, which allow us to highlight the difference between the losses.
\end{enumerate}
Also, we compared all losses in a more complex setup considered in~\cite{enders2024risk}, where the authors propose to use risk-averse RL to increase the robustness of the learned policy against distribution shifts.
We disclose all \underline{technical details} in Appendix~\ref{ap:details}.

\vspace{-2mm}
\subsection{Portfolio Optimization and Hedging}\label{seq:exp:fin}
\vspace{-2mm}
Consider the problem of optimal stock trading in several setups.
The state space is represented by the stock price augmented with timestamp: $s = (t, S_t) \in \{0,\ldots,T\} \times \bbR$.
Each time, an agent can buy or sell any amount of stock, so the action space $\cA = \bbR$ is continuous.
We consider the discrete-time Bachelier model~\cite{bachelier1900theorie} of stock price dynamics%
% ~\footnote{Although the model may seem too simplistic compared to those used in practice~\cite{gatheral2011volatility}, it allows us to evaluate our method \emph{quantitatively}. The extension to a more complex model falls beyond the scope of the current work.}%
, so the price increments are independent and normally distributed.
Let $Z_t \sim \cN(\mu, \sigma^2),\ t=1,\ldots,T$ be the iid Gaussian variables with mean $\mu$ and variance $\sigma^2$.
The state transition law is: $s' = (t+1,\ S_t + Z_{t+1}),\ t=0,\ldots,T-1$, so the state transition law is independent of the action taken.
The initial stock price is $S_0$, so $s^0 = (0, S_0)$.
The reward function is specified differently for each setup.

We parametrize $\pi_\phi(s)$ and $\tilde V^{\pi}_\theta(s)$ as multi-layer perceptrons.
We use the TD(0) learning with function approximation~\cite{sutton2018reinforcement} to learn the $V$-functions.
Authors in~\cite{hau2023entropic} prove that the optimal policy is deterministic in this case, so the action is a non-random output of $\pi_\phi(s)$.
To learn the optimal policy, we minimize the following objective:
\begin{equation}
    \cL(\phi) = \E[s']{\alpha^{-1} \exp\bigl\{ -\alpha \bigl( r(s, \pi_\phi(s), s') + \tilde V_{\theta}^{\pi}(s') - \tilde V_{\theta}^{\pi}(s) \bigr) \bigr\}},
\end{equation}
which estimates the gradient of $\alpha^{-1} \exp\bigl\{-\alpha \left(\tilde Q^\pi\bigl(s, \pi_\phi(s)\bigr) - \tilde V^\pi_{\theta}(s) \right)\bigr\}$ using one TD(0) sample.
It learns the correct policy, since the function $\alpha^{-1} \exp\{-\alpha x\}$ is monotonically decreasing.

\begin{figure}[t]
    \vspace{-5mm}
    \centering
    \begin{subfigure}[b]{0.48\textwidth}
        \centering
        \includegraphics[width=\linewidth]{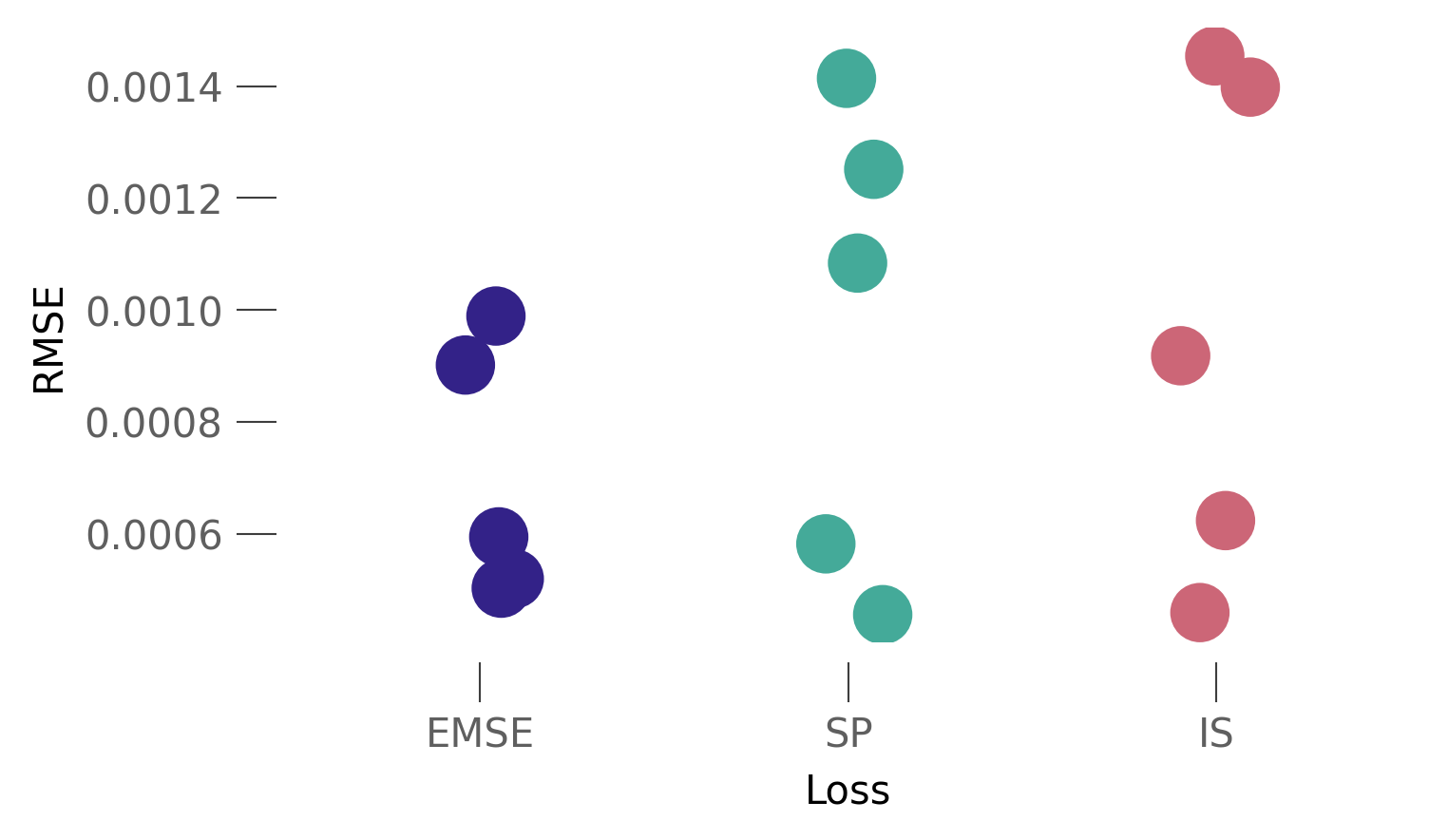}
        \caption{Gaussian reward}
        \label{fig:gauss}
    \end{subfigure}
    \hfill
    \begin{subfigure}[b]{0.48\textwidth}
        \centering
        \includegraphics[width=\linewidth]{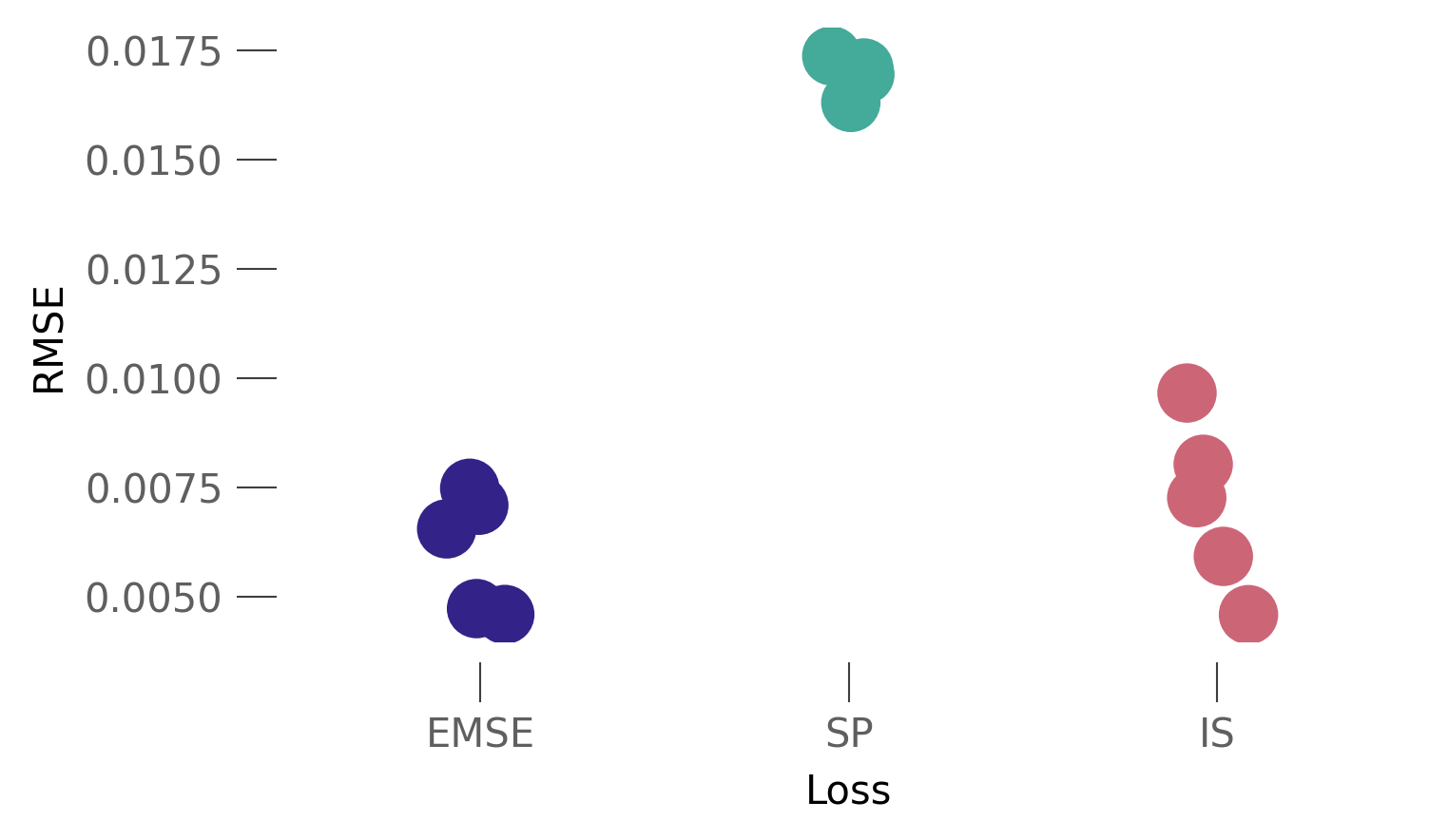}
        \caption{Quadratic reward}
        \label{fig:quad}
    \end{subfigure}
    \caption{%
        Error in learning the obtained approximation of $V^*$ in the Gaussian and quadratic cases.
        Each experiment was run five times with different random seeds.
        In the Gaussian case, losses perform on par.
        Loss~\eqref {eq:soft_plus} does not learn the correct value function for the non-Gaussian return.
    }
    \label{fig:toy}
\end{figure}

\paragraph{Analytically Tractable Cases}\label{sec:exp:toy}

In the first experiment we set $\mu > 0$ and $r(s, a, s') = a(S_{t+1} - S_t),\ t=0,\ldots,T-1$, so the rewards come from stock trading solely.
Return $R^\pi(s^0) = \sum_{t=0}^{T-1} a^t Z_{t+1}$ is distributed normally as a sum of Gaussian random variables, so the application of objective~\eqref{eq:soft_plus} is mathematically sound here.
The optimal $V$-function and optimal policy can be derived analytically (see Appendix~\ref{ap:sol}):
\begin{equation}
\pi^*(s) = \frac{\mu}{\alpha \sigma^2}, \quad V^*(s) = \frac{\mu^2 (T - t)}{2\alpha \sigma^2}.
\end{equation}

Next, we set $\mu = 0$ and we consider the reward function of the form:
\begin{equation}
    r(s, a, s') =
    \begin{cases}
        a(S_{t+1} - S_t),   & t=0,\ldots,T -2, \\
        a(S_T - S_{T-1}) + g(S_T), \quad g(x) = \frac12 (x - S_0)^2 & t=T-1.
    \end{cases}
\end{equation}
It is similar to the previous one, except the agent receives a quadratic reward at the last moment.
The return is not Gaussian anymore, so the loss~\eqref{eq:soft_plus} learns the incorrect value function.
An analytical solution exists in this case:
\begin{equation}
    \pi^*(s) = \alpha (S_t - S_0), \quad
    V^*(s) = -\frac{-\alpha(S_t - S_0)^2 + (T-t) \log\left(1 - \alpha \sigma^2\right)}{2\alpha}.
\end{equation}

We compare the~\eqref{eq:is_loss} loss with the alternatives in Figure~\ref{fig:toy}.
We use RMSE between the learned value function and the analytical solution $\text{RMSE} = \sqrt{\E[s]{(V^*(s) - V^\pi_\theta(s))^2}}$.
Note that the expectation does not depend on $a$, because the state-transition law is independent of $a$.

\paragraph{Deep Hedging}

\begin{figure}[t]
    % \vspace{-5mm}
    \centering
    \begin{subfigure}[b]{0.48\textwidth}
        \centering
        \includegraphics[width=\linewidth]{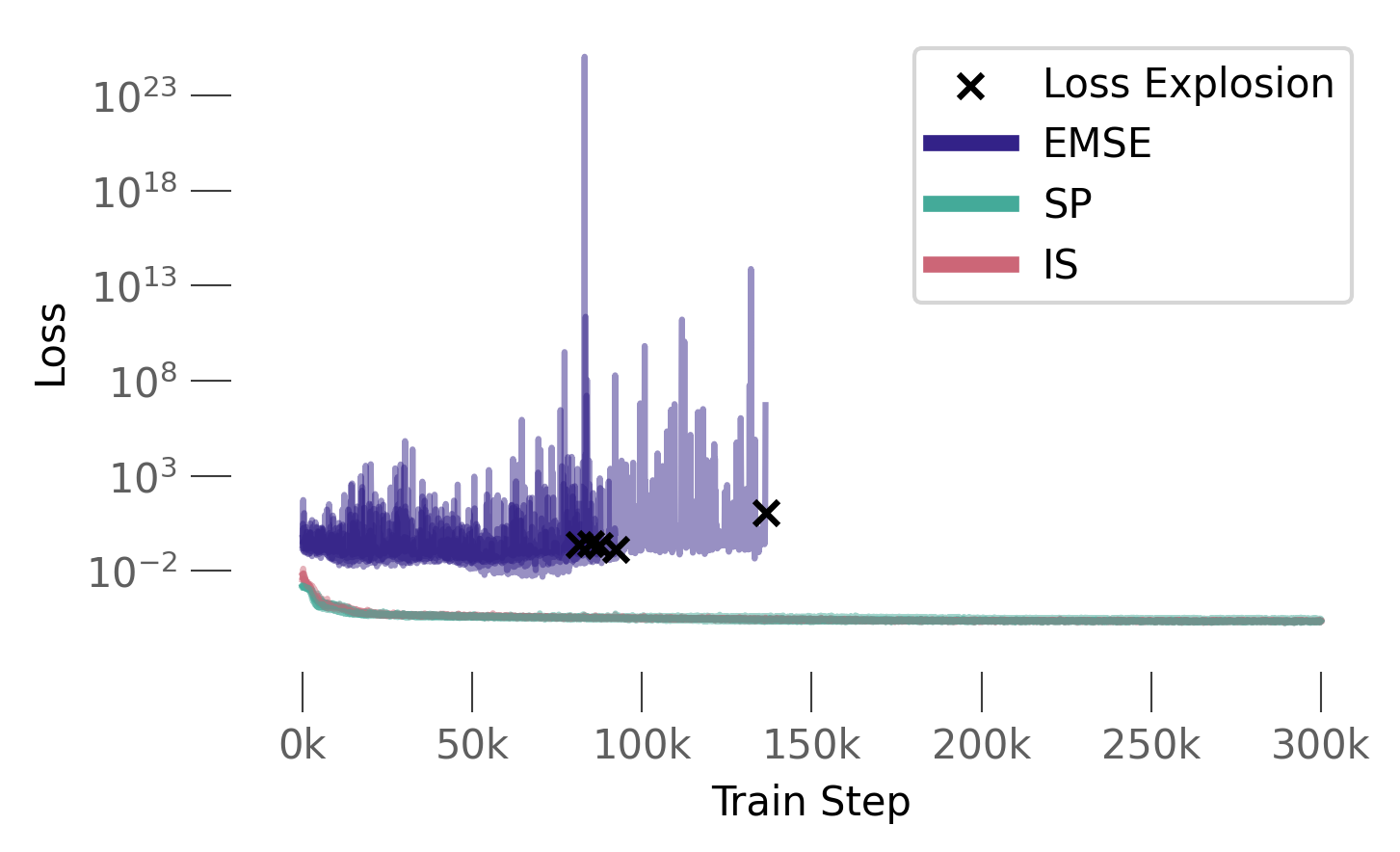}
        \caption{%
            Training process with $\alpha=10$.
            We depict the loss value during training for five random seeds for each loss.
            Objectives~\eqref{eq:soft_plus} and~\eqref{eq:is_loss} converge successfully, while all runs with~\eqref{eq:exp_mse} failed.
        }
        \label{fig:explision}
    \end{subfigure}
    \hfill
    \begin{subfigure}[b]{0.48\textwidth}
        \centering
        \includegraphics[width=\linewidth]{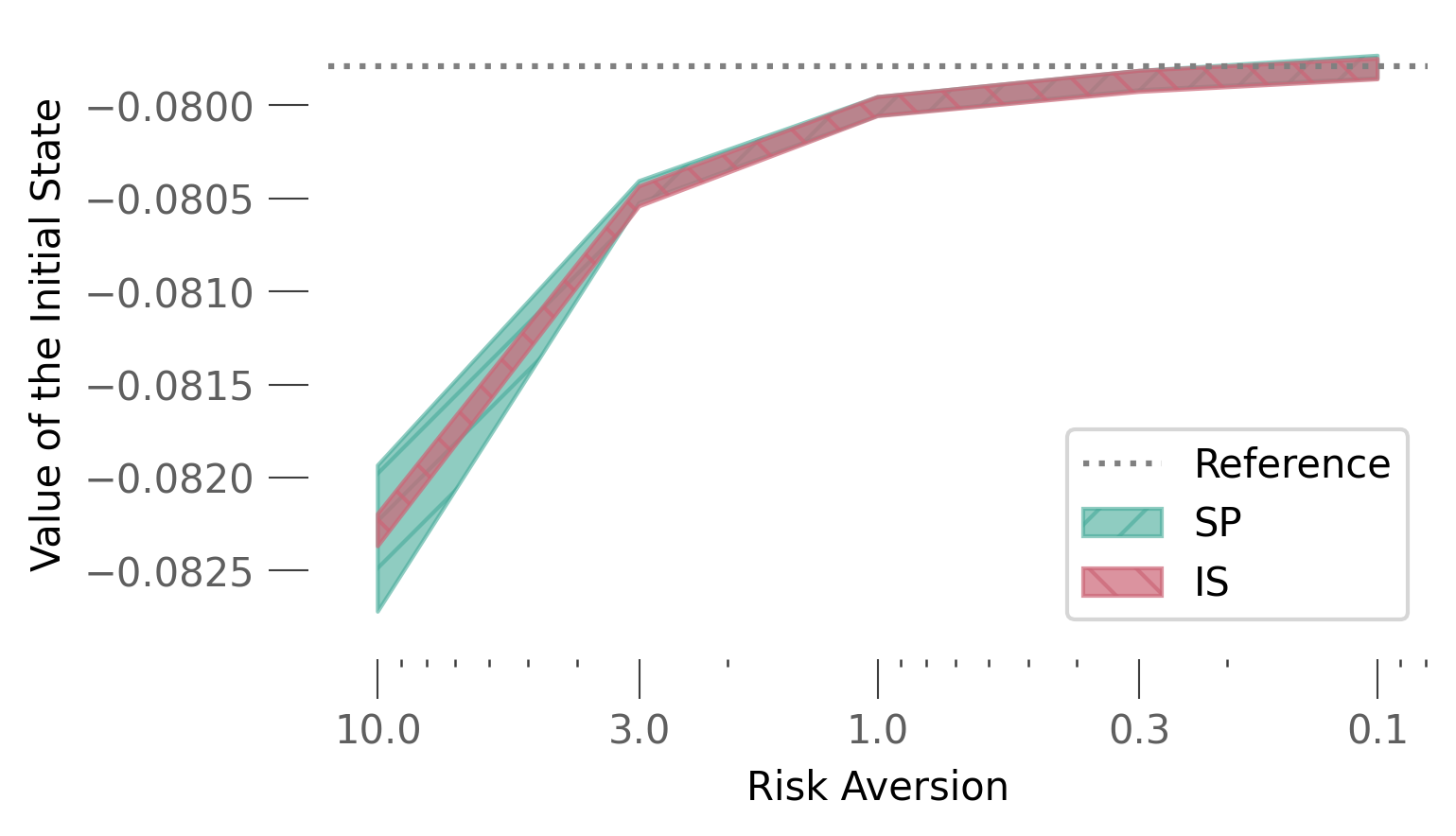}
        \caption{%
            We run~\eqref{eq:soft_plus} and~\eqref{eq:is_loss} five times for each value of risk aversion $\alpha$.
            The filling covers the area $\pm$ 1 standard deviation around the mean value.
            Although losses converge to the theoretical risk-neutral reference, \eqref{eq:is_loss} is more stable for large values of $\alpha$ than~\eqref{eq:soft_plus}.
        }
        \label{fig:convergence}
    \end{subfigure}
    \caption{%
        Loss performance on the Deep Hedging problem~\cite{deep_hedging}.
        Loss~\eqref{eq:is_loss} shows more stable and reliable convergence than the alternatives.
    }
    \label{fig:deep_hedging}
\end{figure}

The European call option is the simplest non-linear derivative contract, which has the payoff of the form $h(x) = \max\{x - K, 0\}$, where $x$ is the stock price at a pre-determined moment and $K$ is the contract parameter, called the strike price~\cite{follmer2011stochastic}.
We consider the following reward function:
\begin{equation}
    r(s, a, s') =
    \begin{cases}
        a(S_{t+1} - S_t),   & t=0,\ldots,T -2, \\
        a(S_T - S_{T-1}) - h(S_T),\quad h(x) = \max\{x - K, 0\} & t=T-1.
    \end{cases}
\end{equation}
Similar problems are widely considered in the financial literature~\cite{deep_hedging, murray2022deep, cao2021deep, kolm2019dynamic, stoiljkovic2023applying}.
The main goal is to calculate the price of the derivative, $\tilde V^*\!\left(s^0\right)$ in our notations.
In our case, the closed-form solution is available only for the risk-neutral case ($\alpha = 0$): $V^*\!\left(s^0\right) = \sigma \sqrt{T / 2\pi}$~\cite{bachelier1900theorie, choi2022black}.
Interestingly, in the risk-neutral case, every policy for which the expected return is well defined mathematically is optimal, because the price process is a martingale.

We show the results in Figure~\ref{fig:deep_hedging}.
First, Figure~\ref{fig:explision} supports our speculations~\eqref{eq:emse_taylor} about the unstable nature of~\eqref{eq:exp_mse} loss.
Second, Figure~\ref{fig:convergence} shows that~\eqref{eq:soft_plus} and~\eqref{eq:is_loss} succeeded in converging to the theoretical risk-neutral reference.
However, the~\eqref{eq:soft_plus} loss shows higher variance across random seeds than~\eqref{eq:is_loss} loss.
We speculate that the primary cause is the non-convex nature of~\eqref{eq:soft_plus}, so the optimization procedure can get stuck in a local minimum.
Also, the return is not Gaussian in this case (although close), so the usage of~\eqref{eq:soft_plus} is not justified.

\vspace{-2mm}
\subsection{Risk-Averse Soft Actor-Critic (RSSAC) for Robust Combinatorial Optimization}\label{sec:exp:rssac}
\vspace{-2mm}

\begin{figure}[t]
    \vspace{-1cm}
    \centering
    \begin{subfigure}[b]{0.48\textwidth}
        \centering
        \includegraphics[width=\linewidth]{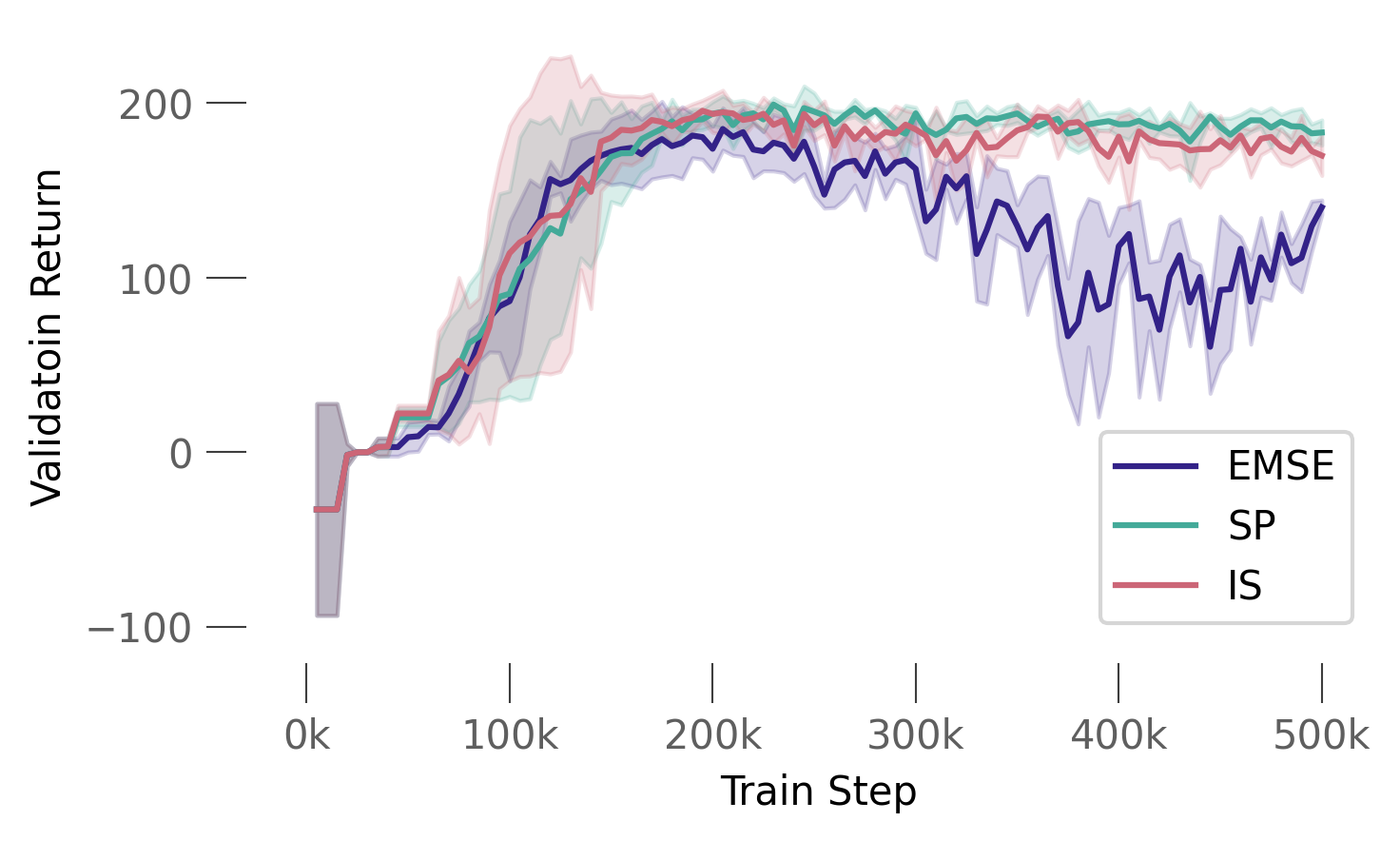}
        \caption{%
            Validation performance of risk-sensitive SAC under undiscounted returns ($\gamma=1$).
        }
        \label{fig:rssac_ud}
    \end{subfigure}
    \hfill
    \begin{subfigure}[b]{0.48\textwidth}
        \centering
        \includegraphics[width=\linewidth]{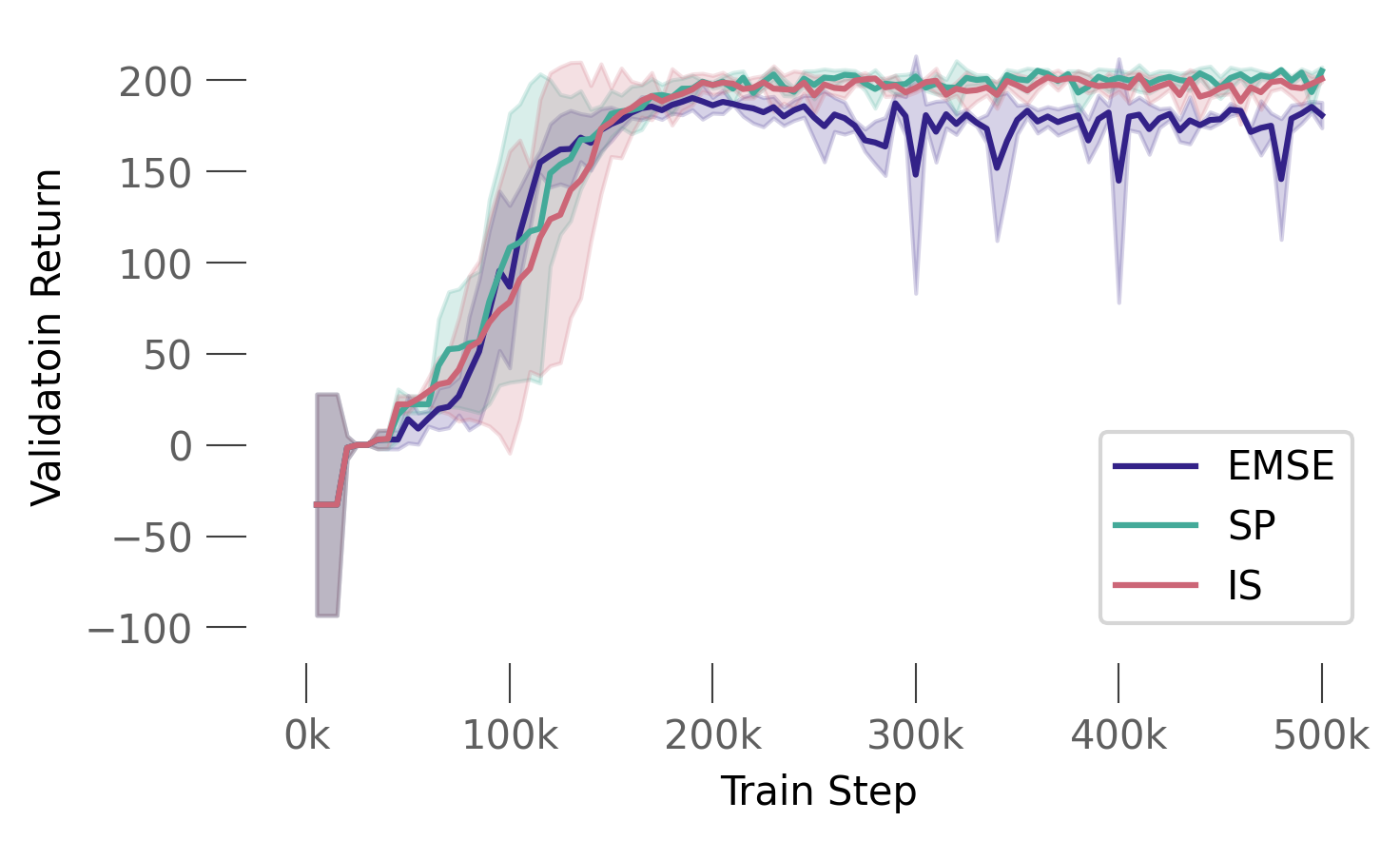}
        \caption{%
            Validation performance of RSSAC under discounted returns ($\gamma=0.99$).
        }
        \label{fig:rssac_d}
    \end{subfigure}
    \caption{%
        Loss performance on the RSSAC problem~\cite{enders2024risk}. Learning curves depict the mean validation return during the training process.
        Each line represents the average over three random seeds, with shaded areas indicating $\pm 1$ standard deviation.
        The~\eqref{eq:exp_mse} loss destabilizes training.
    }
    \label{fig:rssac}
\end{figure}

This experiment aims to show that loss~\eqref{eq:is_loss} can act as a performance-enhancing drop-in replacement of known losses in complex RL algorithms.
We adopt the experimental setup of~\cite{enders2024risk}, which resembles the warehouse management problem.
The environment is a $5 \times 5$ grid.
Each time, items randomly and independently appear in grid cells according to some probability distribution unknown to the agent.
The agent can move up, down, left, right, or stay.
Any movement costs $-1$ to the agent.
When the agent reaches the cell with an item, it picks it.
If the agent delivers the item to the specific cell, it receives a $+15$ reward.
If an item is not picked during some period after its spawn, it disappears.
The agent can carry at most one item at a time.
The duration of episodes is constant.

The authors study the problem of learning a policy robust to distribution shifts.
They propose the risk-averse soft actor-critic algorithm with exponential utility to learn such policies.
The algorithms learns a $Q$-function as a critic.
The authors rely on the approximate equality $\E{X^\gamma} \approx \E{X}^\gamma$, where $\gamma$ is a discount factor, and derive the Bellman equation of the form:
\begin{equation}\label{eq:rssac_bellman}
    Q^\pi_\theta(s, a) = \EXD[s', a']{r(s, a, s') + \gamma \kappa \cH(\pi_\phi\bigl(\cdot \mid s)\bigr) + \gamma Q_{\theta^-}^\pi(s', a')},
\end{equation}
where $\cH \bigl( \pi_\phi(\cdot \mid s) \bigr)$ is the entropy of the policy $\pi_\phi$ and $\kappa$ is the entropic regularization coefficient.
The authors note the unstable nature of~\eqref{eq:exp_mse} loss and propose to regress $Q_\theta^\pi(s, a)$ on sampling-based estimation of the rhs of~\eqref{eq:rssac_bellman}.
They replace the expectation over $s'$ with a single $s'$ from the replay buffer and directly compute the expectation over the next actions to estimate the regression target.
As noted by the authors, this results in a biased estimation of the $Q$-function.
Nevertheless, they do it because they do not aim to recover the correct $Q$-function, but to learn the close-to-optimal policy.

We run the proposed soft actor-critic algorithms with minor modifications.
Instead of relying on the direct estimation of the rhs of~\eqref{eq:rssac_bellman}, we learn the $Q$-function using the unbiased objective.
We compare the losses in Figure~\ref{fig:rssac}.
The figure shows the validation return during the training process.
To measure the robustness, the return is computed with the probability of item appearance different from the one used during training.
The objective~\eqref{eq:exp_mse} destabilizes the training process.
The loss~\eqref{eq:is_loss} performs on par with~\eqref{eq:soft_plus} and consistently outperforms the~\eqref{eq:exp_mse} objective in complex RL algorithms and environments.

\vspace{-2mm}
\section{Conclusion}
\vspace{-2mm}

This paper considers the Itakura-Saito loss first proposed in~\cite{murray2022deep}, a simple loss function to learn the value function in risk-averse MDPs.
We proved that the minimizer of this loss is indeed the correct value function and derived the corresponding SA rule.
Numerical experiments show that alternatives either destabilize training or do not recover the correct value function.
Itakura-Saito loss can be used as a drop-in replacement in complex RL algorithms.
% \vspace{-1em}
% \paragraph{Limitations}
% The proposed methodology only applies to the exponential utility.
% Though we do not see this as a problem, since many methods rely on solving the exponential MDP~\cite{hau2023entropic, henry2019martingale, nutz2023martingale}.
% \vspace{-1em}
% \paragraph{Broader impact}
% While our work is methodological in nature and tested in synthetic settings, we believe that developing reliable mathematical tools for optimization under uncertainty is necessary to make AI more acceptable and dependable in real-world applications.
% This paper presents work whose goal is to advance the field of RL.
% There are many potential societal consequences of our work, none which must be specifically highlighted here.

% \begin{ack}
% Grant number blablalba
% \end{ack}

\printbibliography[notkeyword={phys}]

%%%%%%%%%%%%%%%%%%%%%%%%%%%%%%%%%%%%%%%%%%%%%%%%%%%%%%%%%%%%

% \input{checklist}

%%%%%%%%%%%%%%%%%%%%%%%%%%%%%%%%%%%%%%%%%%%%%%%%%%%%%%%%%

\newpage
\appendix

\section{Analytical Solutions}\label{ap:sol}

We provide here the closed‐form derivations for the ground‐truth benchmarks reported in the experiments (cf. Section~\ref{sec:exp} in the main text).
All results are obtained under the discrete‐time Bachelier dynamics
\[
   S_{t+1}=S_t+Z_{t+1},\qquad
   Z_{t+1}\sim\mathcal N(\mu,\sigma^{2}),\quad t=0,\dots,T-1 ,
\]
and use the exponential‐utility \emph{certainty equivalent}
\[
   \widetilde{\mathbb E}_{\!\alpha}[X]\;=\;
   -\frac1\alpha\log\mathbb E[\exp\{-\alpha X\}],\qquad \alpha>0 .
\]
For any Gaussian variable \(G\sim \mathcal N(m,v)\) one has
\[ \mathbb E\exp\bigl\{-\alpha G\bigr\}]=
   \exp\bigl\{-\alpha\mathbb E[G]+\tfrac{\alpha^{2}}2\operatorname{Var}(G)\bigr\} = \exp\bigl\{-\alpha m+\tfrac{\alpha^{2}}2v\bigr\}
\] a fact used repeatedly below.

\subsection{Pure trading with Gaussian reward}
\label{ap:gauss}

Consider a single‐period reward \( r_t=a_t(S_{t+1}-S_t)=a_tZ_{t+1} \).
Since $(Z_t)$ are i.i.d.\ Gaussians, the optimal \emph{deterministic policy} $\pi^{\!*}$ is time-independent and the value function does not depend on  \(S_t\). For any fixed action \(a\) the one-step certainty-equivalent return is
\[
  \widetilde{\mathbb E}_{\!\alpha}[aZ_{t+1}]
  = -\frac1\alpha\log\mathbb E[\exp\{-\alpha aZ_{t+1}\}]
  = \mu a-\tfrac12\alpha a^{2}\sigma^{2}.
\]
Maximising this quadratic over $a$ gives
\(
a_t^{\!*}= \frac{\mu}{\alpha\sigma^{2}}
\)
and the corresponding single‐step optimum
\(
\max_a \widetilde{\mathbb E}_{\!\alpha}[aZ_{t+1}]=
   \frac{\mu^{2}(T-t)}{2\alpha\sigma^{2}}.
\)
Summing over $T-\tau$ periods yields the value function
\( V^{\!*}(s_t) =\mu^{2}(T-t)/(2\alpha\sigma^{2})\).
Hence, the optimal policy is constant in time and satisfies
\[
   \boxed{\;
      \pi^{\!*}(s_t)=\frac\mu{\alpha\sigma^{2}},
      \quad
      V^{\!*}(s_t)=\frac{\mu^{2}(T-t)}{2\alpha\sigma^{2}},\;}
\]
as quoted in Eq.~(14) in Sec.~\ref{sec:exp:toy}. No additional integrability condition is required here because $\exp(-\alpha a Z)$ is integrable for all $\alpha$.

\subsection{Trading with a quadratic terminal penalty (\(\mu=0\))}

Let \(x_t =S_t-S_0\) be the centred price.
Rewards are \(r_t=a_tZ_{t+1}\) for \(t<T-1\), while at $t=T-1$ the agent additionally incurs a terminal penalty
\(
r_{T-1}=a_{T-1}Z_{T}-\frac12x_{T}^{2}
\).
We look for a solution of the form with boundary conditions:
\[
V_t(x) = -\frac{1}{2}K_tx^2 + C_t, \quad K_T = 1,\quad C_T=0.
\]
Since \(x_{t+1}=x_t+Z_{t+1}\) conditional on \(x_t\), \(x_{t+1}\) is Gaussian, so
\begin{multline*}
   \mathbb E \Bigl[\exp\Bigl\{-\alpha\{aZ_{t+1}-\tfrac12K_{t+1}(x_t+Z_{t+1})^{2}+C_{t+1}\}\Bigr\}\Bigr]
   = \\
   = \exp\Bigl\{-\alpha C_{t+1}+\frac{\alpha K_{t+1}}2x_t^{2}\Bigr\}
     (1-\alpha K_{t+1}\sigma^{2})^{-\frac12}
     \exp\!\left(
        -\frac{\alpha\sigma^{2}}{2(1-\alpha K_{t+1}\sigma^{2})}(a-K_{t+1}x_t)^{2}
     \right).
\end{multline*}

By direct calculation of the certainty equivalent at each step, one finds that the exponential inside the expectation is maximized when \(a=K_{t+1}x_t\), so the optimal policy is
\[
      a_t^{\!*}=K_{t+1}x_t
\]
and the maximized Bellman update is dynamic recursion
\[
   V_t(x)=C_{t+1}-\frac12K_{t+1}x^{2}
          +\frac1{2\alpha}\log(1-\alpha K_{t+1}\sigma^{2}).
\]
Matching coefficients gives \(K_t=K_{t+1}\) and \(K_T=1\), hence \(K_t\equiv1\) for all \(t\). Meanwhile, the scalars \(C_t\) satisfy
\(
   C_t=C_{t+1}+\frac1{2\alpha}\log(1-\alpha\sigma^{2}),
\)
with boundary \( C_T=0\), which solves to
\[
C_t=\frac{T-t}{2\alpha}\log(1-\alpha\sigma^{2}).
\]
The recursion is well-defined only if $\alpha\sigma^{2}<1$,
ensuring the Gaussian moment-generating function exists. Under the above condition the optimal policy and value are
Provided \( \alpha\sigma^{2}<1\) (so that the required moment‐generating function is well‐defined), the optimal policy and value function are
\[
   \boxed{\;
      \pi^{\!*}(s_t)=S_t-S_0,\;
      V^{\!*}(s_t)=
        -\frac12(S_t-S_0)^{2}
        +\frac{T-t}{2\alpha}\log(1-\alpha\sigma^{2}),
      \quad\alpha\sigma^{2}<1,\;}
\]
which coincides with Eq.~(16) in Sec.~\ref{sec:exp:toy}.

% \subsection{Risk-neutral reference for deep hedging}
% \label{ap:call}

% When \(\alpha\to0\) the agent’s objective becomes risk‐neutral.
% In this case, any policy for which the expected return is well-defined is optimal.
% The expected return is given by:
% \[
% \E{R^\pi} = \E{ \max \{S_T - S_0, 0\}} = \sigma\sqrt{\tfrac{T}{2\pi}}
% \]
% by the Bachelier formula~\cite{bachelier1900theorie, choi2022black}.

\section{Experimental Details}\label{ap:details}
\subsection{Portfolio Optimization and Hedging}
Throughout all experiments in this section, we take $T = 10$.
The code is written in pure \texttt{PyTorch~2.7.0}.
We approximate the state-value function and policy with multi-layer perceptrons with Mish activation and 2 hidden layers, 64 neurons each.
We used the Adam optimizer (\texttt{lr=1e-4}, $\beta_1 = $ 0.99, $\beta_2 = $ 0.999) and applied the following learning rate schedule:
\begin{itemize}
\item During the first 1k iterations, the learning rate grows linearly with a start factor of $0.01$.
\item It then remains constant for the next 49k (149k for Deep Hedging experiments) iterations.
\item Afterwards, we apply cosine decay with \texttt{T\_max} equal to 50k (150k for Deep Hedging).
\end{itemize}

Each training batch has a size of 1024, with gradient values clipped at $1$ and the gradient norms clipped at $10$. Each experiment is repeated with five independent seeds that affect weight initialization and mini-batch sampling. The other parameters are listed in the table:

\begin{table}[h]
  \centering
  \label{tab:momcond}
  \begin{tabular}{rccc}
    \toprule
     Experiment & $\mu$ & $\sigma$ & $\alpha$ \\
    \midrule
     Gaussian return & $0.03$ &  \multirow{3}{2em}{$\frac{0.2}{\sqrt{10}}$} & $1$  \\
     Quadratic penalty & \multirow{2}{0.5em}{$0$} &  & $100$  \\
     Deep Hedging & & & $\{0.1,0.3,1,3,10\}$ \\
    \bottomrule
  \end{tabular}
\end{table}

Our experiments typically run on an A100 GPU, with up to three runs in parallel on a single device. For the Gaussian and quadratic tasks, we use 100k training iterations, requiring between 25 and 65 minutes of wall‐clock time depending on concurrent jobs. The Deep Hedging setup generally requires 50 minutes to 2 hours for each run, owing to its longer 300k iteration schedule.

\subsection{Risk-Averse Soft Actor-Critic (RSSAC) for Robust Combinatorial Optimization}

We employ the open‐source code provided by  the authors of~\cite{enders2024risk}%
\footnote{\url{https://github.com/tumBAIS/RiskSensitiveSACforRobustDRLunderDistShifts}.}
and leave all network, replay-buffer and optimiser hyper-parameters unchanged except for the three modifications below.
\begin{enumerate}
    \item We took $\beta = -0.1$, which is $\alpha = 0.1$ in our notations;
    \item The discount factor is $\gamma \in \{0.99, 1\}$;
    \item The loss function is simpler than in~\cite{enders2024risk} and do not involve computing the expectation over the next actions.
\end{enumerate}

Every RSSAC run (500 k environment steps) takes $\approx$\,2 h on a single A100 GPU. No concurrent jobs are scheduled on the same device.

\section{Statement and Proof of Proposition~\ref{prop:is}}\label{ap:proof}
Before stating the result we recall the Itakura–Saito loss \eqref{eq:is_loss}
\begin{equation}
    \cL^\text{IS}(\theta) \stackrel{\text{def}}{=} \alpha^{-2} \E[(s, a, s')]{\exp\bigl\{ \alpha \delta_{\tilde V}(\theta) \bigr\} - \alpha \delta_{\tilde V}(\theta) - 1}
\end{equation}
and
\begin{equation}\label{eq:delta_recall}
    \delta_V(\theta) = V^\pi_\theta(s) - r(s, a, s') -  V^\pi_{\theta^-}(s').
\end{equation}
\begin{proposition*}
    Assume the following hold:
    \begin{enumerate}
        \item\label{as:target} The target network weights $\theta^-$ are a copy of the main network weights $\theta$, $V^\pi_{\theta^-}(s') = \mathrm{stop\ gradient}\bigl(V_\theta^\pi(s')\bigr)$ in~\eqref{eq:delta_recall};
        \item\label{as:ex} Both $\E{\exp \bigl\{ -\alpha r(s, a, s') - \alpha \tilde V^\pi_{\theta^-}(s') \bigr\}}$ and $\E{r(s, a, s') + \tilde V^\pi_{\theta^-}(s')}$ exist;
        \item\label{as:att} $\cL^{ \text{IS}}$ attains its minimum at $\theta^\star$.
    \end{enumerate}
    Then, $\tilde V^\pi_{\theta^*}$ satisfies the risk–averse Bellman equation~\eqref{eq:evv}:
\begin{equation}
\tilde V^\pi(s) = \EXD[a, s']{r(s, a, s') + \tilde V^\pi(s')}.
\end{equation}
\end{proposition*}\begin{proof}
    First, note that
    \[
    \cL^\text{IS}(\theta) = \alpha^{-2}\E[s, a, s']{d_\text{IS} \left(\ \exp \left\{ -\alpha r(s, a, s') -\alpha \tilde V_\theta^\pi(s') \right\},\ \exp\left\{ -\alpha\tilde V_\theta^\pi(s) \right\}\ \right)},
    \]
    where $\alpha^{-2}$ is a scaling coefficient.
    Thanks to the assumption~\ref{as:ex}, one can apply Theorem 1 in~\cite{tan2013optimal}, so for the minimizer it holds:
    \[
    \exp\left\{ -\alpha\tilde V_{\theta^*}^\pi(s) \right\} = \E[a, s']{\exp \left\{ -\alpha r(s, a, s') -\alpha \tilde V_{\theta^*}^\pi(s') \right\}}.
    \]
    Taking the logarithm of both sides, dividing by $-\alpha$ and recalling the definition of $\EXD{\cdot}$ we obtain:
    \[
    \tilde V_{\theta^*}^\pi(s) = \EXD[a, s']{r(s, a, s') + \tilde V_{\theta^*}^\pi(s')}.
    \]
\end{proof}

\section{Geometric and Field-Theoretic Interpretations of the IS Loss}

This appendix explores the deeper mathematical structure behind the Itakura-Saito (IS) loss and its regularization. While the main paper focuses on empirical performance and optimization theory, here we reinterpret the IS loss in the language of variational calculus and field theory.

We show that, under regularization, the IS loss induces an energy functional with features reminiscent of classical scalar field theories. In particular, we analyze its behavior in the linearized regime, where it yields a propagator with long-range correlations and, in the high-target limit, reveals an emergent conformal symmetry. These structures illuminate why the IS loss promotes global consistency and smoothness in prediction.

We then speculate on a broader analogy: drawing inspiration from the AdS/CFT correspondence in theoretical physics, we suggest that the robustness and coherence induced by IS-trained models may be seen as a form of holographic encoding — where local losses reflect and preserve global structure.

No advanced physics background is required to follow this appendix — only a curiosity for the beautiful human adventure of understanding the universe through mathematics, symmetry, and learning.

\subsection{Conformal Symmetry and Long-Range Correlations Induced by Itakura-Saito Loss}

To understand the structural properties of the IS loss beyond pointwise optimization, we adopt a variational viewpoint by constructing an \emph{action functional}. This approach is standard in physics and calculus of variations: instead of minimizing a discrete loss over samples, we define a functional that integrates a local loss density over a continuous domain. In this setting, the learned function $\phi(x)$ becomes a field, and the loss becomes an energy functional whose minimizers reflect both local fit and global coherence. This field-theoretic perspective allows us to uncover the long-range smoothing and scale-invariance properties implicit in the IS loss when combined with regularization.

We analyze the field-theoretic formulation of the IS loss with smoothness regularization and show how conformal symmetry emerges in the continuum limit.

\subsection{Two-Point Correlation from the IS-Induced Action}

We consider the action functional:
\begin{equation}
\mathcal{S}[\phi] = \int dx \left[ \lambda \left( \frac{d\phi}{dx} \right)^2 + \frac{y(x)}{\phi(x)} - \log\left(\frac{y(x)}{\phi(x)}\right) - 1 \right]
\end{equation}

To study long-range correlations, we simplify and linearize around a constant input $y(x) = y_0$, assuming:
\begin{equation}
\phi(x) = y_0 + \varepsilon(x), \quad \varepsilon(x) \ll y_0
\end{equation}

Expanding the IS potential term:
\begin{align*}
\frac{y_0}{\phi} &= \frac{1}{1 + \varepsilon / y_0} \approx 1 - \frac{\varepsilon}{y_0} + \left(\frac{\varepsilon}{y_0}\right)^2 \\
\log\left(\frac{y_0}{\phi}\right) &= -\log\left(1 + \frac{\varepsilon}{y_0}\right) \approx -\left( \frac{\varepsilon}{y_0} - \frac{1}{2} \left(\frac{\varepsilon}{y_0}\right)^2 \right)
\end{align*}

So the potential becomes:
\begin{equation}
V_{\text{IS}}(y_0, \phi) \approx \frac{1}{2} \left( \frac{\varepsilon}{y_0} \right)^2
\end{equation}

The linearized action reads:
\begin{equation}
\mathcal{S}[\varepsilon] \approx \int dx \left[ \lambda \left( \frac{d\varepsilon}{dx} \right)^2 + \frac{1}{2 y_0^2} \varepsilon^2(x) \right]
\end{equation}

This is a standard Gaussian field theory with mass $m^2 = \frac{1}{2y_0^2}$.

\subsection{Propagator and Emergence of Conformal Limit}

The propagator satisfies:
\begin{equation}
\left( -\lambda \frac{d^2}{dx^2} + \frac{1}{2 y_0^2} \right) G(x) = \delta(x)
\end{equation}

Here, $\delta(x)$ denotes the \emph{Dirac delta function}, a generalized distribution satisfying $\int \delta(x) f(x)\, dx = f(0)$ for any smooth test function $f$. It models a pointwise source and is used to define Green's functions in field theory.

Solving in 1D yields:
\begin{equation}
G(x) = \frac{y_0}{\sqrt{2\lambda}} e^{-\frac{|x|}{\sqrt{2\lambda} y_0}}
\end{equation}

In the conformal limit $y_0 \to \infty$ or $\lambda \to 0$:
\begin{equation}
G(x) \sim \frac{1}{|x|}
\end{equation}

This reflects:
\begin{itemize}
  \item Absence of intrinsic length scale,
  \item Power-law correlation decay,
  \item Conformal symmetry and long-range propagation.
\end{itemize}

\subsection{Implications for Optimization Dynamics}

We now show how IS loss, being scale-invariant, improves optimization convergence.

Let $\theta \in \mathbb{R}^d$ be model parameters, with outputs $f(\theta) \in \mathbb{R}^n$, and targets $y$. Assume gradient descent:
\begin{equation}
\theta_{t+1} = \theta_t - \eta \nabla L(\theta_t)
\end{equation}

We compare:
\begin{itemize}
  \item MSE: $L_{\text{MSE}}(\theta) = \frac{1}{2} \| f(\theta) - y \|^2$
  \item IS: $L_{\text{IS}}(\theta) = \sum_i \left( \frac{y_i}{f_i(\theta)} - \log\left(\frac{y_i}{f_i(\theta)}\right) - 1 \right)$
\end{itemize}

\textbf{Conditioning and Hessian Geometry:}

\begin{itemize}
\item MSE: Hessian is $H = J^\top J$ where $J = \partial f / \partial \theta$.
\item Large spread in $y$ leads to poor conditioning.
\item IS: Penalizes relative differences $f(\theta)/y$.
\item Built-in normalization leads to better conditioned Hessian and stable steps.
\end{itemize}

\textbf{Local Approximation in 1D:}

Let $f = y + \varepsilon$ with small $\varepsilon$. Then:
\begin{align*}
L_{\text{IS}}(f) &= \frac{y}{f} - \log\left(\frac{y}{f}\right) - 1 \\
&\approx \frac{1}{2} \left( \frac{\varepsilon}{y} \right)^2
\end{align*}

Hence IS behaves locally like a rescaled MSE:
\begin{equation}
L_{\text{IS}}(f) \approx \frac{1}{2 y^2} \varepsilon^2
\end{equation}

Large $y$ $\Rightarrow$ lower weight, small $y$ $\Rightarrow$ higher weight. This performs implicit preconditioning, akin to natural gradients.

Natural gradients arise in information geometry as an improvement over standard (Euclidean) gradients. Instead of computing updates in the raw parameter space, the natural gradient rescales the direction of steepest descent using the inverse Fisher information matrix. This aligns updates with the intrinsic curvature of the loss landscape, leading to faster convergence and better conditioning. In the context of IS loss, the local reweighting of errors by $1/y^2$ mimics this effect: low-output regions receive stronger updates, similar to how natural gradients emphasize directions of low Fisher variance~\cite{amari1998natural}.

\section{Conformal Invariance Improves Statistical Conditioning}

In this appendix, we formalize the intuition that conformal invariance, as induced by the Itakura-Saito (IS) loss, statistically improves the conditioning of the optimization problem. We do so by analyzing the distribution of Hessian spectra under a Gaussian prior over models.

\textit{Remark.} This theorem highlights a novel geometric interpretation of scale-invariance: it not only regularizes the optimization surface but also statistically preconditions the curvature — providing a mathematical basis for the empirical advantages of IS loss.

We propose a theoretical explanation for the improved optimization behavior observed when using the Itakura-Saito (IS) loss, through the lens of statistical conditioning and symmetry constraints.

\begin{theorem*}[Improved Spectral Conditioning under Conformal Invariance]
Let $\phi \sim \mathcal{G}$ be a Gaussian ensemble of models (e.g., fields or neural networks), and let $\mathcal{G}_{\text{conf}} \subset \mathcal{G}$ be the subset of models such that the IS-induced action $\mathcal{S}_{\text{IS}}[\phi]$ is conformally invariant.

Define a spectral conditioning measure $\lambda(\phi)$, such as the variance or entropy of the eigenvalue spectrum of the Hessian at $\phi$. Then:
\begin{equation}
\mathbb{E}_{\phi \in \mathcal{G}_{\text{conf}}}[\text{Var}(\lambda(\phi))] < \mathbb{E}_{\phi \in \mathcal{G}}[\text{Var}(\lambda(\phi))]
\end{equation}
\end{theorem*}

\begin{proof}[Sketch of proof]
The IS loss induces a field theory with conformal symmetry in the limit $y_0 \to \infty$, leading to scale-free correlations. When restricting to $\mathcal{G}_{\text{conf}}$, the underlying symmetry enforces statistical regularity in the structure of the Hessian. In contrast, over the full Gaussian ensemble $\mathcal{G}$, more arbitrary fluctuations are allowed. By known results from random matrix theory and information geometry, imposing such symmetries reduces the variance and entropy of the spectral distribution. This leads to improved average conditioning.
\end{proof}

This supports the empirical observation that the IS loss improves the landscape geometry for gradient descent by constraining optimization trajectories to a submanifold of better-conditioned models.

\subsection{Discussion and Related Work}

While a general formal proof connecting conformal invariance to better conditioning is still open, similar ideas appear in:
\begin{itemize}
\item Theoretical physics, where scale-invariant theories exhibit smoother correlation functions and long-range order \cite{henkel1999conformal, nishimori2010elements}.
\item Studies of the conformal bootstrap, which show how scale invariance constrains fluctuations and narrows operator spectra \cite{poland2018conformal}.
\item Random matrix theory and kernel methods, where flattened spectra correspond to improved generalization.
\item Recent numerical studies that identify conformal symmetry in real physical transitions as a signature of underlying flatness and universality \cite{zhu2023uncovering, emergent2023quantum}.
\end{itemize}

These connections reinforce the insight that conformal invariance leads to a better distributed Hessian spectrum — measurable through variance and entropy — and thus enhances the robustness and convergence of gradient-based learning.

\subsection{The particular case of flattened spectra and RMT }

We talk about flattened spectra when Hessians or kernel matrices have lower spectral variance or more uniform eigenvalue distribution. This is usually associated with improved generalization in both kernel methods and deep networks. In random matrix theory (RMT), this corresponds to ensembles where the eigenvalue density concentrates, reducing the effect of high-curvature directions that may cause overfitting. In kernel methods~\cite{belkin2019reconciling} and~\cite{jacot2018ntk} show that flatter spectra lead to more stable interpolation and better generalization in the overparameterized regime. Similar ideas appear in deep learning: models trained with flatter loss landscapes, reflected in the Hessian spectrum (e.g., fewer large outliers), often exhibit better generalization~\cite{yao2018hessian, sagun2016eigenvalues}. In our setting, the conformal IS loss effectively preconditions the Hessian, reducing its spectral variance,  thus echoing these theoretical insights.

The flattening of the Hessian spectrum induced by conformal invariance also resonates with results from random matrix theory (RMT). In high-dimensional models, RMT provides a statistical framework to describe the eigenvalue distribution of Hessians, Fisher matrices, or kernel operators. A common benchmark is the Marchenko–Pastur law, which characterizes the spectrum of sample covariance matrices in the absence of structure. Deviations from this law — such as heavy tails, outliers, or sharp spectral peaks — often signal overfitting or poor generalization~\cite{louart2018random, pennington2018emergence}. In contrast, flatter or more regular spectra (e.g., those with lower variance and fewer extreme outliers) tend to reflect better generalization performance. In our setting, the IS-induced conformal symmetry reduces the spectral variance of the Hessian, effectively steering the system toward a more stable and “bulk-like” spectrum, thus aligning with favorable RMT regimes.

\subsection{CFT Robustness and the Holographic Analogy}

In the high-target limit $y_0 \to \infty$, we have shown that the IS loss regularized by a smoothness penalty induces an effective action
\begin{equation}
S[\epsilon] = \int dx\, \lambda \left( \frac{d\epsilon}{dx} \right)^2
\end{equation}

This corresponds to a massless scalar field theory — a conformal field theory (CFT) in one dimension — with long-range power-law correlations $G(x) \sim 1/|x|$ and no intrinsic length scale.

Beyond its mathematical elegance, this structure echoes the foundational role of CFTs in the AdS/CFT correspondence \cite{maldacena1999large}. In this duality, a gravitational theory defined in a $(d+1)$-dimensional Anti-de Sitter (AdS) bulk is fully determined by a $d$-dimensional CFT living on its boundary. Perturbations and dynamics in the bulk geometry are encoded in boundary correlators and operator insertions of the CFT.

\paragraph{Robustness as Boundary Consistency.}
Within this holographic framework, the robustness and stability of the boundary CFT are essential for the well-posedness of the dual gravitational theory. Small fluctuations on the boundary propagate into coherent and physically meaningful bulk geometries. Similarly, in our context, we observe that the IS loss — by enforcing conformal invariance in the large-$y_0$ regime — leads to stable, globally consistent learning dynamics. The long-range coherence of the predictions resembles the behavior of a holographic boundary theory robustly encoding bulk structure.

\paragraph{A Speculative Correspondence.}
We may interpret the IS-trained model as a system where:
\begin{itemize}
    \item The prediction layer behaves like a boundary CFT, enforcing smooth, scale-invariant structure.
    \item The internal representation space (e.g., hidden layers or latent variables) acts as a discrete, dynamically evolving “bulk,” whose geometry is regularized through this induced boundary behavior.
\end{itemize}

This analogy opens a speculative but intriguing avenue: \emph{robust generalization in learning systems may reflect a kind of holographic encoding, where local losses induce structured global behavior}. The IS loss, in this view, acts as a holographically robust preconditioner: it ensures that perturbations do not remain confined, but are smoothed in a globally consistent manner.

\paragraph{Future Directions.}
While our model remains in 1D and is far from a full-fledged CFT, let alone a holographic dual, the structural parallels motivate future exploration. It may be fruitful to investigate:
\begin{itemize}
    \item Learning systems that approach RG fixed points under IS-like training, mirroring conformal fixed points.
    \item Architectures where latent representations exhibit AdS-like metrics, optimized via boundary-consistent losses.
    \item Explicit bulk-boundary decompositions in model design, inspired by holographic renormalization.
\end{itemize}

This connection between optimal loss design, conformal symmetry, and holographic stability invites a deeper geometric and dynamical understanding of learning systems.

\paragraph{Conclusion}

The IS loss with regularization is mathematically equivalent to a conformal scalar field theory. It induces power-law correlations and better-conditioned optimization landscapes. This explains both its empirical stability and its ability to propagate information globally in high-risk or low-signal regimes.

Thus, \textbf{scale-invariant losses improve convergence} by embedding learning into the geometry of critical systems with long-range structure~\cite{cardy1996scaling,francesco2012conformal,zinn2021quantum,peskin2018introduction}.

\printbibliography[title={Additional References},keyword={phys}]

\end{document}